\documentclass[12pt]{article}

\usepackage{amsmath,amsfonts,amssymb,amsthm,graphicx,color,subfigure}
\usepackage[normalem]{ulem}
\usepackage{xcolor}
\usepackage[utf8]{inputenc}
\usepackage{verbatim}
\usepackage{algorithmic}
\usepackage{algorithm}
\usepackage{verbatim}
\usepackage{url}
\usepackage{amsfonts}
\usepackage{tikz}
\usetikzlibrary{matrix}
\usetikzlibrary{arrows}
\usetikzlibrary{decorations.pathreplacing}

\DeclareGraphicsExtensions{.pdf,.png,.jpg}

\def\sec{\textsc{SparseMix}}

\renewcommand{\algorithmiccomment}[1]{\bgroup\hfill//~#1\egroup}

\def\1{\mathds{1}}

\def\X{\mathcal{X}}

\def\Y{\mathcal{Y}}

\def\Q{\mathcal{Q}}

\newtheorem{theorem}{Theorem}[section]

\theoremstyle{definition}
\newtheorem{definition}{Definition}[section]
\newtheorem{remark}{Remark}[section]

\newtheorem{example}{Example}[section]

\newcommand\xor{\mathrm{xor}}
\newcommand\D{\mathrm{D}}
\newcommand\cost{\mathrm{cost}}
\newcommand\code{\mathrm{code}}

\title{Efficient mixture model for clustering of sparse high dimensional binary data}
\author{Marek \'Smieja \and Krzysztof Hajto \and Jacek Tabor}
\date{\normalsize $^{1}$Faculty of Mathematics and Computer Science\\ Jagiellonian University\\ \L{}ojasiewicza 6, 30-348 Krakow, Poland}

\begin{document}

\maketitle

\begin{abstract}
In this paper we propose a mixture model, {\sec}, for clustering of sparse high dimensional binary data, which connects model-based with centroid-based clustering. Every group is described by a representative and a probability distribution modeling dispersion from this representative. 

In contrast to classical mixture models based on EM algorithm, {\sec}:
\begin{itemize}
\item[-] is especially designed for the processing of sparse data,
\item[-] can be efficiently realized by an on-line Hartigan optimization algorithm, 
\item[-] is able to automatically reduce unnecessary clusters.
\end{itemize} We perform extensive experimental studies on various types of data, which confirm that {\sec} builds partitions with higher compatibility with reference grouping than related methods. Moreover, constructed representatives often better reveal the internal structure of data.
\end{abstract}

\section{Introduction}

Sparse high-dimensional binary representations became very popular in various domains including text mining, cheminformatics, biology, etc. \cite{juan2002use, klekota2008chemical, franti2003classification}. Binary features are usually used to indicate whether an object contains a predefined pattern or satisfies a given rule, e.g. one can represent a sentence by its words (set-of-words) \cite{baker1998distributional} or identify a chemical compound by its chemical structures (fingerprint) \cite{ewing2006novel}. Therefore, efficient processing and learning from such data is of great practical importance.


In this paper we introduce a version of model-based clustering, {\sec}, which efficiently processes high-dimensional sparse binary data\footnote{An implementation of {\sec} algorithm, together with some example data sets, is available on GitHub: \url{https://github.com/hajtos/SparseMIX}.}. Our model splits the data into groups which can be efficiently compressed by a collection of coding algorithms; each algorithm is designed for encoding elements within a single cluster (Section \ref{sec:model}). Since the code-lengths directly depend on the associated probability distribution, the elements generated from similar distributions are grouped together and consequently we obtain the effect of model-based clustering. In contrast to classical mixture models using Bernoulli variables or latent trait models, our model is designed for sparse data and can be optimized by an on-line Hartigan algorithm, which converges faster and finds better solutions than batch procedures like EM  (Section \ref{sec:algo}).

{\sec} builds a bridge between mixture models and centroid-based clustering, and describes every cluster by its representative (a single vector characterizing the most popular cluster patters) and probability distribution modeling dispersion from this representative. The relationship between the form of the representative and the associated cluster distribution is controlled by an additional parameter of the model. By placing a parameter selection problem on solid mathematical ground, we show that we can move from a model providing the best compression rate of data to the one obtaining high speed performance (Section \ref{sec:algo} and Theorem \ref{thm:compression}). Our method can automatically discover the number of groups by introducing a well-justified cost of cluster identification. We present a theoretical and experimental analysis how the number of clusters depends on the main characteristics of the data set (Example \ref{ex:reduction}). 

\begin{figure}[t]
\centering
\includegraphics{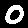} \includegraphics{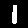} 
\includegraphics{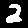} \includegraphics{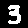} 
\includegraphics{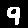} \includegraphics{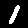} 
\includegraphics{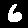} \includegraphics{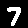} 
\includegraphics{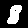} \includegraphics{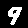}
\caption{Representatives of handwritten digits from MNIST database produced by {\sec}.}
\label{fig:intro}
\end{figure}

The paper contains extensive experiments performed on various types of data, including text corpora, chemical and biological data sets, as well as the MNIST image database. The results show that {\sec} gives higher compatibility with reference partition than existing methods based on mixture models and similarity measures. Most of clusters representatives obtained by {\sec} on MNIST data set correspond to distinct digits, which confirms high quality of its results, see Figure \ref{fig:intro}. Its running time is significantly lower than in related model-based algorithms and comparable to methods implemented in the Cluto package, which are optimized for processing large data sets. 

The paper is organized as follows. The next section contains a brief description of related clustering methods. The {\sec} model is introduced in Section \ref{sec:model}. Section \ref{sec:algo} presents a practical implementation of our method. Experiments are included in Section \ref{sec:experiments} and finally the conclusion is given.

\section{Related work}

A lot of clustering methods are expressed in a geometric framework, where similarity between objects is defined with the use of the Euclidean metric, e.g. $k$-means \cite{macqueen1967some}. Although a geometric description of clustering can be insightful for continuous data, it becomes less informative in the case of high dimensional binary (or discrete) vectors, where the Euclidean distance is not natural. To adapt these approaches to binary data sets, the authors consider, for instance, $k$-medoids or k-modes \cite{huang1998extensions, chan2004optimization} with dissimilarity measure designed for this special type of data, such as Hamming, Jaccard or Tanimoto measures \cite{li2005general}. Evaluation of possible dissimilarity metrics for categorical data can be found in \cite{dos2015categorical, bai2011novel}. To obtain a more flexible structure of clusters one can also use hierarchical methods \cite{zhao2002evaluation} or density-based clustering \cite{wen2002query}. 

Model-based clustering \cite{mclachlan2004finite}, where data  is modeled as a sample from a parametric mixture of probability distributions, is commonly used for grouping continuous data using Gaussian models, but has also been adapted for processing binary data. In the simplest case, the probability model of each cluster is composed of a sequence of independent Bernoulli variables (or multinomial distributions) describing the probabilities on subsequent attributes \cite{celeux1991clustering, juan2002use}. Since many attributes usually are statistically irrelevant and independent of true categories, they may be removed or associated with small weights \cite{graham2006unsupervised, bouguila2010multivariate}. This partially links  mixture models with subspace clustering of discrete data \cite{yamamoto2015clustering, chen2016soft}. Since the use of multinomial distributions formally requires independence of attributes, different smoothing techniques were proposed, such as applying Dirichlet distributions as a prior to the multinomial \cite{bouguila2009discrete}. Another version of mixture models for binary variables tries to maximize the probability that data points are generated around cluster centers with the smallest possible dispersion \cite{celeux1991clustering}. This technique is closely related to our approach, however, our model allows for using any clusters representatives (not only cluster centers), is significantly faster due to the use of sparse coders and can automatically deduce the number of clusters.

A mixture of latent trait analyzers is a specialized type of mixture model for categorical data, where a continuous univariate of a multivariate latent variable is used to describe the dependence in categorical attributes \cite{vermunt2007multilevel, gollini2014mixture}. Although this technique recently received high interest in the literature \cite{langseth2009latent, cagnone2012factor}, it is potentially difficult to fit the model, because the likelihood function involves an integral that cannot be evaluated analytically. Moreover, its use is computationally expensive for large high dimensional data sets \cite{tang2015}.

Information-theoretic clustering relies on minimizing the entropy of a partition or maximizing the mutual information between data and its discretized form \cite{li2004entropy, tishby2000information, dhillon2003information}. Although both approaches are similar and can be explained as a minimization of coding cost, the first creates ``hard clusters'', where an instance is classified to exactly one category, while the second allows for soft assignments \cite{Strouse_DIB}. Information-theoretic clustering was combined with model selection criteria as MDLP (minimum description length principle) \cite{rissanen1985minimum} or MML (minimum message length) \cite{wallace1968information} to add a model complexity to the clustering objective function \cite{barbara2002coolcat}. The authors of \cite{celeux1991clustering} established a close connection between entropy-based techniques for discrete data and model-based clustering using Bernoulli variables. In particular, entropy criterion can be formally derived using a classification maximum likelihood.

To the best authors knowledge, neither model-based clustering nor information-theoretic methods have been optimized for processing sparse high-dimensional binary data. Our method can be seen as a combination of $k$-medoids with model-based clustering (in the sense that it describes a cluster by a single representative and a multivariate probability model), which is efficiently realized for sparse high-dimensional binary data. It is worth to mention the Cluto package \cite{karypis2002cluto}, which is a practical clustering software designed especially for processing sparse high dimensional data. The Cluto package is built on a sophisticated multi-level graph partitioning engine and offers many different criteria that can be used to derive both partitional, hierarchical and spectral clustering algorithms.

\section{Clustering model} \label{sec:model}

The goal of clustering is to split the data into groups that contain elements characterized by similar patters. In our approach the elements are similar if they can be efficiently compressed by the same algorithm. We begin this section with presenting a model for compressing elements within a single cluster. Next, we combine these partial encoders and define a final clustering objective function. 

\subsection{Compression model for a single cluster} \label{sec:single}

Let us assume that $X \subset \{0,1\}^D$ is a data set (cluster) containing $D$-dimensional binary vectors. We implicitly assume that $X$ represents sparse data, i.e. the number of positions occupied by non-zero bits is relatively low.

A typical way for encoding such data is to directly remember the values at each coordinate \cite{barbara2002coolcat, li2004entropy}. Since, in practice, $D$ is often large, this straightforward technique might be computationally inefficient. Moreover, due to the sparsity of data, positions occupied by zeros are not very informative while the less frequent non-zero bits carry substantial knowledge. Therefore, instead of remembering all the bits of every vector, it might be more convenient to encode positions occupied by non-zero values. It occurs that this strategy can be efficiently implemented by on-line algorithms.

To realize the aforementioned goal, we first select a representative (prototype) $m \in \{0,1\}^D$ of a cluster $X$. Next, for each data point $x = (x_1,\ldots,x_D) \in X$ we construct a corresponding vector 
$$
\xor(x,m) = (|x_1-m_1|, \ldots,|x_D-m_D|) \in \{0,1\}^D
$$ 
of differences with $m$. If a representative is chosen as the most probable point of a cluster (mean of a cluster), then the data set of differences will be sparser on average than the original data set $X$. An efficient way for storing such sparse data relies on encoding the numbers of coordinates with non-zero bits. Concluding, the original data $X$ is compressed by remembering a representative and encoding resulting vectors of differences in an efficient manner, see Figure \ref{fig:coding2}.

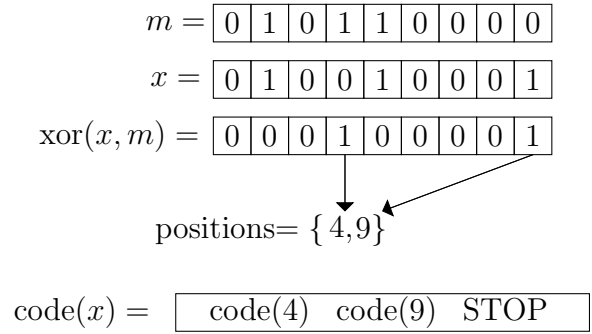
\begin{figure}[t]
\centering
\begin{tikzpicture}
\node(bit 1) at (0.5,1.75){$m=$};
\draw (1,2.0) rectangle (1.5,1.5);
\node(bit 1) at (1.25,1.75){0};
\draw (1.5,2.0) rectangle (2.0,1.5);
\node(bit 1) at (1.75,1.75){1};
\draw (2.0,2.0) rectangle (2.5,1.5);
\node(bit 1) at (2.25,1.75){0};
\draw (2.5,2.0) rectangle (3.0,1.5);
\node(bit 1) at (2.75,1.75){1};
\draw (3.0,2.0) rectangle (3.5,1.5);
\node(bit 1) at (3.25,1.75){1};
\draw (3.5,2.0) rectangle (4.0,1.5);
\node(bit 1) at (3.75,1.75){0};
\draw (4.0,2.0) rectangle (4.5,1.5);
\node(bit 1) at (4.25,1.75){0};
\draw (4.5,2.0) rectangle (5,1.5);
\node(bit 1) at (4.75,1.75){0};
\draw (5.0,2.0) rectangle (5.5,1.5);
\node(bit 1) at (5.25,1.75){0};

\node(bit 1) at (0.5,1.0){$x=$};
\draw (1,1.25) rectangle (1.5,0.75);
\node(bit 1) at (1.25,1.0){0};
\draw (1.5,1.25) rectangle (2.0,0.75);
\node(bit 1) at (1.75,1.0){1};
\draw (2.0,1.25) rectangle (2.5,0.75);
\node(bit 1) at (2.25,1.0){0};
\draw (2.5,1.25) rectangle (3.0,0.75);
\node(bit 1) at (2.75,1.0){0};
\draw (3.0,1.25) rectangle (3.5,0.75);
\node(bit 1) at (3.25,1.0){1};
\draw (3.5,1.25) rectangle (4.0,0.75);
\node(bit 1) at (3.75,1.0){0};
\draw (4.0,1.25) rectangle (4.5,0.75);
\node(bit 1) at (4.25,1.0){0};
\draw (4.5,1.25) rectangle (5,0.75);
\node(bit 1) at (4.75,1.0){0};
\draw (5.0,1.25) rectangle (5.5,0.75);
\node(bit 1) at (5.25,1.0){1};

\node(bit 1) at (-0.25,0.25){$\xor(x,m)=$};
\draw (1,0) rectangle (1.5,0.5);
\node(bit 1) at (1.25,0.25){0};
\draw (1.5,0) rectangle (2.0,0.5);
\node(bit 1) at (1.75,0.25){0};
\draw (2.0,0) rectangle (2.5,0.5);
\node(bit 1) at (2.25,0.25){0};
\draw (2.5,0) rectangle (3.0,0.5);
\node(bit 1) at (2.75,0.25){1};
\draw (3.0,0) rectangle (3.5,0.5);
\node(bit 1) at (3.25,0.25){0};
\draw (3.5,0) rectangle (4.0,0.5);
\node(bit 1) at (3.75,0.25){0};
\draw (4.0,0) rectangle (4.5,0.5);
\node(bit 1) at (4.25,0.25){0};
\draw (4.5,0) rectangle (5,0.5);
\node(bit 1) at (4.75,0.25){0};
\draw (5.0,0) rectangle (5.5,0.5);
\node(bit 1) at (5.25,0.25){1};

\node(bit 1) at (1.35,-1.0){positions$=\{$};
 \draw[
        -triangle 90,
        line width=0.02mm,
        postaction={draw, line width=0.01cm, shorten >=0.1cm, -}
    ] (2.75,0) -- (2.75,-0.75);
\node(bit 1) at (2.7,-1.0){$4,$};
 \draw[
        -triangle 90,
        line width=0.02mm,
        postaction={draw, line width=0.01cm, shorten >=0.1cm, -}
    ] (5.25,0) -- (3.25,-0.75);
\node(bit 1) at (3.1,-1.0){$9\}$};
\node(bit 1) at (-0.75,-2.1){$\code(x)=$};
\draw (0.5,-2.35) rectangle (6.0,-1.85);
\node(bit 1) at (3.2,-2.1){$\code(4)$ \enspace $\code(9)$ \enspace STOP};
\end{tikzpicture}
	\caption{Sparse data coding.}
	\label{fig:coding2} 

\end{figure}

We now precisely follow the above idea and calculate the cost of coding in this scheme, which will be the basis of our clustering criterion function. Let the distribution at the $i$-th coordinate of $x \in X$ be described by a Bernoulli random variable taking value 1 with a probability $p_i \in [0,1]$ and 0 with a probability $(1-p_i)$, i.e. $p_i = P(x_i = 1)$. For a fixed $T \in [0,1]$, we consider a representative $m = m(T) = (m_1,\ldots,m_D)$ defined by
$$
m_i = \left\{ \begin{array}{ll} 0, & p_i \leq T, \\ 1, & p_i > T,\end{array}\right.
$$
Although a representative $m(T)$ depends on $T$, we usually discard this parameter and simply write $m$, when $T$ is known from the context. The $i$-th coordinate of $X$ is more likely to attain value 1, if $p_i > \frac{1}{2}$ and, in consequence, for $T = \frac{1}{2}$ the representative $m$ coincides with the average vector (most probable point) of $X$.

Given a representative $m$, we consider the differences $\xor(x,m)$, for $x \in X$, and denote such a data set by 
$$
\D_m(X) = \{\xor(x,m): \text{ } x \in X\}. 
$$
The probability $q_i = q_i(T)$ of bit 1 at the $i$-th position in $D_m(X)$ equals 
$$
q_i = \left\{ \begin{array}{ll} p_i, & p_i \leq T, \\ 1-p_i, & p_i > T. \end{array}\right.
$$
Let us notice that $ q_i \leq p_i$, for $T \geq \frac{1}{2}$, which makes $\D_m(X)$ sparser than $X$, see Figure \ref{fig:diffprob}. 

\begin{figure}
\centering
\includegraphics[width=0.95\textwidth]{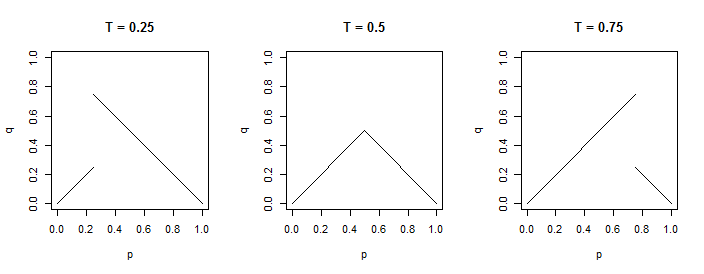}
\caption{Relation between probabilities $p_i$ and $q_i$.} \label{fig:diffprob}
\end{figure}

To design an efficient coder for sparse data, let us consider a probability distribution $\Q = \Q(T) = \{Q_1,\ldots,Q_D\}$ on the set of coordinates $\{1,\ldots,D\}$ describing a distribution of positions with non-zero bits in $\D_m(X)$. A number $Q_i$ is a conditional probability that the $i$-th position holds value 1 if at least one coordinate in a vector is non-zero, i.e.
$$
Q_i = P\left(|x_i-m_i| = 1 \, \middle| \, \sum_{i=1}^D x_i > 0\right) = \frac{q_i}{Z},
$$
where $Z = Z(T) = \sum_{j=1}^D q_j$ is a normalization factor. In practice, it is convenient to restrict the attention to the support of $\Q$ (non-zero probabilities), because we usually have $Q_i = 0$ for most $i$.

The Shannon entropy theory states that the code-lengths in an optimal prefix-free coding depend strictly on the associated probability distribution \cite{cover2012elements}. Given a distribution $\Q$ of positions with bit 1 it is possible to construct at most $D$ codes, each with the length\footnote{in the limiting case} $-\log Q_i$ (we generate codes only for $Q_i > 0$). The short codes correspond to the most frequent symbols, while the longest ones are related with rare objects. Given an arbitrary element $d=(d_1,\ldots,d_D) \in \D_m(X)$ we encode its non-zero coordinates and obtain that its compression requires
$$
\sum_{\substack{i: d_i = 1 \\ \,\,\, Q_i > 0}} -\log Q_i
$$
bits. 

This leads to the {\sec} objective function for a single cluster, which gives the average cost of compressing a single element of $X$ by our sparse coder:
\begin{definition} (\em{one cluster cost function})
Let $X \subset \{0,1\}^D$ be a data set and let $T \in [0,1]$ be fixed. {\sec} objective function for a single cluster is given by\footnote{We put: $0\cdot \log 0 = 0$}:
\begin{equation} \label{eq:ourCost}
\cost_T(X) = \cost(\D_m(X))  = \sum_{i = 1}^D q_i (-\log Q_i).
\end{equation}
\end{definition}
Observe that, given probabilities $p_1,\ldots,p_D$, the selection of $T$ determines the form of $m$ and $\D_m(X)$.

\begin{remark}
To be able to decode the initial data set, we would also need to remember the probabilities $p_1,\ldots,p_D$ determining the form of the representative $m$ and the corresponding probability $\Q$ used for constructing the codewords. These are the model parameters, which in practical coding scheme are stored once in the header. Since they do not affect the asymptotic value of data compression, we do not include them in the final cost function\footnote{Nevertheless. these probabilities should be accounted in model selection criteria as AIC or BIC.}. 

Moreover, to reconstruct the original data we should distinguish the encoded representations of subsequent vectors. It could be realized by reserving an additional symbol for separating two encoded vectors or by remembering the number of non-zero positions in every vector. Although this is necessary in the coding task, it is less important for clustering and therefore we decided not to include it in the definition.
\end{remark}

The following theorem shows that $T=\frac{1}{2}$ provides the best compression rate of a single cluster. 
\begin{theorem} \label{thm:compression}
Let $X \subset \{0,1\}^D$ be a data set and let $\frac{1}{2} \leq T_1 \leq T_2 \leq 1$ be fixed. If $m(T_1), m(T_2)$ are two representatives and the mean number of non-zero bits in $\D_{m_1}(X)$ is not lower than 1, i.e. $Z(T_1) \geq 1$, then:
$$
\cost_{T_1}(X) \leq \cost_{T_2}(X).
$$
\end{theorem}
\begin{proof}
See \ref{app:compression}
\end{proof}

Although the best model for compression is given by $T= \frac{1}{2}$, the alternative choices for $T$ might sometimes yield better clustering results. In particular, the greater $T$ is, the faster updates in clustering algorithm we obtain (see Section \ref{sec:algo}), which might be a crucial issue in practical usage.

\subsection{Clustering criterion function}

A single encoder allows us to compress simple data. To efficiently encode real data sets, which usually origin from several sources, it is profitable to construct multiple coding algorithms, each designed for one homogeneous part of data. Finding an optimal set of algorithms leads to a natural division of data, which is a basic idea behind our model. Below, we describe the construction of our clustering objective function, which combines partial cost functions of clusters.

Let us assume that we are given a partition of $X$ containing $k$ groups $X_1,\ldots,X_k$ (pairwise disjoint subsets of $X$ such that $X = X_1 \cup \ldots \cup X_k$), where every subset $X_i$ is described by its own coding algorithm. Observe that to encode an instance $x \in X_i$ by such a multiple coding model one should remember its group identifier and the code of $x$ defined by the $i$-th encoder, i.e.,
\begin{equation}\label{code}
\code(x) = [\code(i), \code_i(x)].
\end{equation}
Such a strategy enables unique decoding, because a retrieved coding algorithm allows subsequently for discovering an instance (see Figure \ref{fig:coding1}). The compression procedure should find a division of $X$ and design  $k$ coding algorithms, which minimize the expected length of code given by \eqref{code}. 

\begin{figure}[t]
\centering
\begin{tikzpicture}
\node(bit 1) at (-0.75,0.5){$\code(x)=$};
\draw (0.5,0.3) rectangle (5.0,0.8);
\node(bit 1) at (2.7,0.5){$\code($i-th encoder$)$};
\draw (5.0,0.3) rectangle (7.6,0.8);
\node(bit 1) at (6.3,0.5){$\code_{i}(x)$};

\draw[decorate,decoration={brace,mirror}]  (0.6,0.05) -- node[below] {encoder (cluster)} (4.9,0.05);
\node(bit 1) at (2.6,-0.75){identification};
\draw[decorate,decoration={brace,mirror}]  (5.1,0.05) -- node[below] {codes defined by} (7.5,0.05);
\node(bit 1) at (6.2,-0.75){$i$-th encoder};
\end{tikzpicture}
	\caption{Multi-encoder model.}
	\label{fig:coding1} 
\end{figure}

The coding algorithms for each cluster are designed as described in previous subsection. More precisely, let $p_i=(p^i_1,\ldots,p^i_D)$ be a vector, where $p_j^i$ is a probability that the $j$-th coordinate in the $i$-th cluster is non-zero, for $i=1,\ldots,k$. Next, given a fixed $T$, for each cluster $X_i$ we construct a representative $m_i = (m^i_1,\ldots,m^i_D)$ and calculate the associated probability distributions $q_i = (q^i_1,\ldots,q^i_D)$ and $\Q_i=\{Q^i_1,\ldots,Q^i_D\}$ on the set of differences $\D_{m_i}(X_i)$. The average code-length for compressing a single vector in the $i$-th cluster is given by (see \eqref{eq:ourCost}):
\begin{equation} \label{eq:costSingle}
\cost_T(X_i) = \cost(\D_{m_i}(X_i))  = \sum_{j = 1}^D q^i_j (-\log Q^i_j).
\end{equation}

To remember clusters identifiers, we again follow Shannon's theory of coding. Given a probability $P_i = P(x \in X_i)$ of generating an instance from a cluster $X_i$ (the prior probability), the optimal code-length of the $i$-th identifier is given by
\begin{equation}\label{eq:costId2}
\cost(i) = -\log P_i. 
\end{equation}
Since the introduction of any new cluster increases the cost of clusters identification, it might occur that maintaining a smaller number of groups is more profitable. Therefore, this model will have a tendency to variate the sizes of clusters and, in consequence, some groups might finally disappear (can be reduced).  

The {\sec} cost function combines the cost of clusters identification with the cost of encoding their elements. To add higher flexibility to the model, we introduce an additional parameter $\beta$, which allows to weight the cost of clusters identification. Specifically, if the number of clusters is known a priori, we should put $\beta = 0$ to prevent from reducing any groups. On the other hand, to encourage the model to remove clusters we can increase the value of $\beta$. By default $\beta = 1$, which gives a typical coding model:
\begin{definition} \label{def:secCost} (\em{clustering cost function})
Let $X = \{0,1\}^D$ be a data set of $D$-dimensional binary vectors and let $X_1,\ldots,X_k$ be a partition of $X$ into pairwise disjoint subsets. For fixed $T \in [0,1]$ and $\beta \geq 0$ the {\sec} clustering objective function equals:
\begin{equation}\label{eq:pecCost}
\cost_{\beta,T}(X_1,\ldots,X_k) = \sum_{i=1}^k P_i \cdot \left(\cost_T(X_i) + \beta \cdot \cost(i)\right),
\end{equation}
where $P_i$ is the probability of a cluster $X_i$, $\cost(i)$ is the cost of encoding its identifier \eqref{eq:costId2} and $\cost_T(X_i)$ is the average code-length of compressing elements of $X_i$ \eqref{eq:costSingle}.
\end{definition}

As can be seen, every cluster is described by a single representative and a probability distribution modeling dispersion from a representative. Therefore, our model can be interpreted as a combination of $k$-medoids with model-based clustering. It is worth to comment that for $T=1$, we always get a representative $m = 0$. In consequence, $\D_0(X) = X$ and a distribution in every cluster is directly fitted to original data. 

The cost of clusters identification allows us to reduce unnecessary clusters. To get more insight into this mechanism, we present the following example. For simplicity, we use $T = \beta = 1$.
\begin{example} \label{ex:reduction}
By $P(p,\alpha,d)$, for $p, \alpha \in [0,1]$ and $d \in \{0,\ldots,D\}$, we denote a $D$-dimensional probability distribution, which generates bit 1 at the $i$-th position with probability:
\begin{equation} \label{eq:pMix}
p_i =  \left\{ 
\begin{array}{ll}
\alpha p, & i =1,\ldots,d,\\
(1-\alpha) p, & i = d+1,\dots,D.
\end{array}
\right.
\end{equation}
Let us consider a data set generated by the mixture of two sources:
\begin{equation}\label{eq:mixData}
\omega P(p,\alpha,d) + (1-\omega) P(p,1-\alpha,d),
\end{equation}
for $\omega \in [0,1]$.

To visualize the situation we can arrange a data set in a matrix, where rows correspond to instances generated by the mixture components, while the columns are related to their attributes:
$$
\begin{aligned}
\begin{array}{r@{}}
    \omega~\{\hspace{\nulldelimiterspace} \\
    1-\omega~\{\hspace{\nulldelimiterspace}
  \end{array}
\Big(\smash{\underbrace{
\begin{array}{c}
    \alpha p\\
    (1-\alpha) p
\end{array}}_{d}}
\smash{\underbrace{
\begin{array}{c}
    (1-\alpha) p\\
    \alpha p\\
\end{array}}_{D-d}}\Big)\\[15pt]
\end{aligned}
$$
The matrix entries show the probability of generating bit~1 at a given coordinate belonging to one of four matrix regions. The parameter $\alpha$ determines how similar are the instances generated from the underlying distributions. For $\alpha = \frac{1}{2}$, both components are identical, while for $\alpha \in \{0,1\}$ we get their perfect distinction. 

\begin{figure}[t]
\centering
\subfigure[]{\label{fig:dimFix} \includegraphics[width=0.3\textwidth]{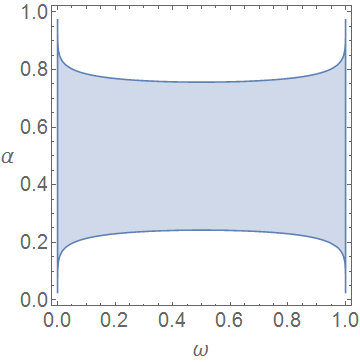}}
\subfigure[]{\label{fig:omegaFix} \includegraphics[width=0.3\textwidth]{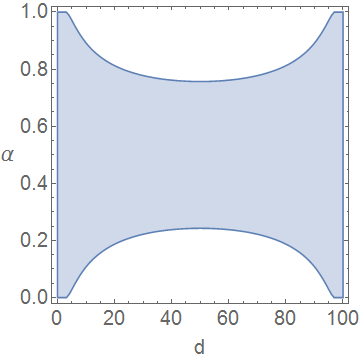}}
\subfigure[]{\label{fig:bothFix} \includegraphics[width=0.3\textwidth]{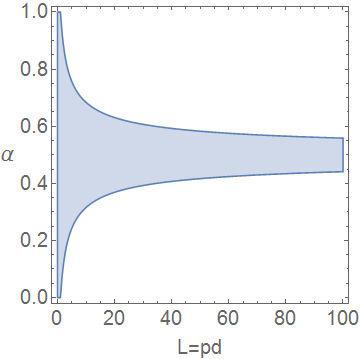}}
\caption{Optimal number of clusters for data generated by the mixture of sources given by \eqref{eq:mixData}. Blue regions show the combinations of mixture parameters which lead to one cluster, while white areas correspond to two clusters. \ref{fig:dimFix} presents the case when every source is characterized by the same number of bits $d = \frac{1}{2} D$, \ref{fig:omegaFix} corresponds to the situation when each source produces the same number of instances, while \ref{fig:bothFix} is the combination of both previous cases.}
\end{figure}

We compare the cost of using a single cluster for all instances with the cost of splitting the data into two optimal groups (clusters are perfectly fitted to the sources generating the data). For the reader's convenience, we put the details of the calculations in \ref{app:reduction}. The analysis of the results is presented below. We consider three cases:
\begin{enumerate}
\item {\bf Sources are characterized by the same number of bits.} The influence of $\omega$ and $\alpha$ on the number of clusters, for a fixed $d = \frac{1}{2}D$, is presented in Figure \ref{fig:dimFix}. Generally, if sources are well-separated, i.e. $\alpha \notin (0.2,0.8)$, then {\sec} will always create two clusters regardless of the mixing proportion. This confirms that {\sec} is not sensitive to unbalanced sources generating the data if only they are distinct. 

\item {\bf Sources contain the same number of instances.} The Figure \ref{fig:omegaFix} shows the relation between $d$ and $\alpha$ when the mixing parameter $\omega=\frac{1}{2}$. If one source is identified by a significantly lower number of attributes than the other ($d << D$), then {\sec} will merge both sources into a single cluster. Since one source is characterized by a small number of features, it might be more costly to encode the cluster identifier than its attributes. In other words, the clusters are merged together, because the cost of cluster identification outweighs the cost of encoding the source elements.

\item {\bf Both proportions of dimensions and instances for the mixture sources are balanced.} If we set equal proportions for source and dimension coefficients, then the number of clusters depends on the average number of non-zero bits in the data $L = p d$, see Figure \ref{fig:bothFix}. For high density of data, we can easily distinct the clusters and, in consequence, {\sec} will end up with two clusters. On the other hand, in the case of sparse data, we use less memory for remembering its elements and the cost of clusters identification grows up with respect to the cost of encoding the elements within the groups.

\end{enumerate}

\end{example}

\section{Fast optimization algorithm} \label{sec:algo}

In this section, we present an efficient on-line algorithm for optimizing the {\sec} cost function. Before that, let us first show how to estimate the probabilities involved in formula \eqref{eq:pecCost}.

\subsection{Estimation of cost function}

We assume that a data set $X \subset \{0,1\}^D$ is split into $k$ groups $X_1,\ldots,X_k$, where $n=|X|$ and $n_i = |X_i|$. Let us denote by 
$$
n^i_j = \sum_{x \in X_i} x_j
$$ 
the number of objects in $X_i$ with the $j$-th position occupied by value $1$. This allows us to estimate the probability $p^i_j$ of bit 1 at the $j$-th coordinate in $X_i$ as 
$$
p^i_j = \frac{n^i_j}{n_i}
$$
and consequently rewrite a representative $m_i = (m^i_1,\ldots,m^i_D)$ of the $i$-th cluster as
$$
m^i_j = \left\{\begin{array}{ll}
0, & \frac{n^i_j}{n_i} \leq T ,\\
1, & \frac{n^i_j}{n_i} > T.
\end{array}
\right.
$$

To calculate the formula for $\cost_T(X_i)$, we first estimate the probability $q_j^i$ of bit $1$ at the $j$-th coordinate in $\D_{m_i}(X_i)$,
$$
q^i_j = \left\{\begin{array}{ll}
\frac{n^i_j}{n_i}, & \frac{n^i_j}{n_i} \leq T ,\\
\frac{n_i - n^i_j}{n_i}, & \frac{n^i_j}{n_i} > T.
\end{array}
\right.
$$
If we denote by 
\begin{equation}\label{eq:Nij}
N^i_j = \left\{\begin{array}{ll}
n_j^i, & \frac{n^i_j}{n_i} \leq T ,\\
n_i-n^i_j, & \frac{n^i_j}{n_i} > T
\end{array}
\right.
\end{equation}
the number of vectors in $\D_{m_i}(X_i)$ with the $j$-th coordinate occupied by bit 1 and by
$$
S_i = \sum_{j=1}^D N^i_j
$$
the total number of non-zero entries in $\D_{m_i}(X_i)$, then we can estimate the probability $Q^i_j$ as:
$$
Q^i_j = \frac{N^i_j}{S_i}.
$$

This allows us to rewrite the cost function for a cluster $X_i$ as
\begin{align*}
\cost_T(X_i) & = \sum_{j=1}^D q^i_j (-\log Q^i_j) \\
& = \sum\limits_{j: p^i_j \leq T} \frac{n^i_j}{n_i} (-\log \frac{N^i_j}{S_i}) + \sum\limits_{j: p^i_j > T} (1-\frac{n^i_j}{n^i}) (-\log \frac{N^i_j}{S_i})\\
& = \frac{1}{n_i}  \sum\limits_{j= 1}^D N^i_j (-\log N^i_j + \log S_i)\\
& = \frac{1}{n_i}  \left(S_i \log S_i + \sum\limits_{j= 1}^D N^i_j (-\log N^i_j) \right).\\
\end{align*}
Finally, since the probability $P_i$ of the $i$-th cluster can be estimated as $P_i=\frac{n_i}{n}$, then the optimal code-length of a cluster identifier equals
$$
\cost(i) = -\log \frac{n_i}{n}.
$$
In consequence, the overall cost function is computed as:
\begin{align*}
\cost_{\beta, T}(X) & = \sum\limits_{i=1}^k \frac{n_i}{n} \left( \beta \cdot \cost(i) + \cost_T(X_i))\right)\\
&= \sum\limits_{i=1}^k \frac{n_i}{n} \left( \beta \cdot (-\log \frac{n_i}{n}) +\frac{1}{n_i} \left[S_i \log S_i + \sum\limits_{j= 1}^D N^i_j (-\log N^i_j) \right] \right)\\
& = \beta \log n + \frac{1}{n}\sum\limits_{i=1}^k  \left( \beta n_i(-\log n_i)  + S_i \log S_i + \sum\limits_{j= 1}^D N^i_j (-\log N^i_j) \right).
\end{align*}

\subsection{Optimization algorithm}

To obtain an optimal partition of $X$, the {\sec} cost function has to be minimized. Since it is not practically feasible to calculate its global minimum, one can use some iterative algorithms to find one of its local minima. In the present paper we adapt a modified version of the Hartigan procedure, which is commonly applied in an on-line version of $k$-means \cite{hartigan1979algorithm}. Although the complexity of a single iteration of Hartigan algorithm is often higher than in batch procedures such as EM, it converges in significantly lower number of iterations and usually finds better minima (see Section \ref{sec:experiments}).

The minimization procedure consists of two steps: initialization and iteration. In the initialization stage, $k \geq 2$ nonempty groups are formed in an arbitrary manner. In the simplest case, it could be a random initialization, but to obtain better results one can also apply a kind of $k$-means++ seeding. In the iteration step the elements are reassigned between clusters in order to minimize the value of the criterion function. Additionally, due to the cost of clusters identification some groups may lose their elements and finally disappear. In practice, a cluster is reduced if its size falls below a given threshold $\varepsilon \cdot |X|$, for a fixed $\varepsilon > 0$.

A detailed algorithm is presented below ($\beta$ and $T$ are fixed):
\begin{flushleft}
\footnotesize
\begin{algorithmic}[1] 
\STATE \textbf{INPUT:}
	\STATE $X \subset \{0,1\}^D$ -- data set
	\STATE $k$ -- initial number of clusters
	\STATE $\varepsilon > 0$ -- cluster reduction parameter
\STATE \textbf{OUTPUT:}\\
	\STATE Partition $\X$ of $X$\\
\STATE \textbf{INITIALIZATION:}\\
		\STATE $\Y = \{Y_1,\ldots,Y_k\}$ -- random partition of $X$ into $k$ groups 
\STATE \textbf{ITERATION:}\\
	\REPEAT
		\FORALL{$x \in X$}
			\STATE $Y_x \leftarrow$ get cluster of $x$
			\STATE $Y \leftarrow \mathrm{arg} \max\limits_{Y \in \Y} \{ \cost_T(Y_x) + \cost_T(Y) - \cost_T(Y_x \setminus \{x\}) - \cost_T(Y \cup \{x\}) \}$ \label{pseudo:rec1}
			\IF{$Y \neq Y_x$}
				\STATE switch $x$ from $Y_x$ to $Y$
				\STATE update probability models of $Y_x$ and $Y$\label{pseudo:rec2}
				\IF{$|Y_x| < \varepsilon \cdot |X|$}
					\STATE delete cluster $Y_x$ and assign its elements to these clusters which minimize the {\sec} cost function
				\ENDIF
			\ENDIF
		\ENDFOR
	\UNTIL{no switch for all subsequent elements of $X$}
\end{algorithmic}
\end{flushleft}
The outlined algorithm is not deterministic and depend on the initial partition. Therefore, the algorithm should be restarted multiple times to avoid getting stuck in local minima.

An efficient implementation of this algorithm requires fast updates of cluster models and recalculation of the {\sec} cost function after switching elements between clusters (see lines \ref{pseudo:rec1} and \ref{pseudo:rec2}). Below, we discuss the details of an efficient recalculation of this cost.

We start with showing how to update $\cost_T(X_i)$, when we add $x$ to a cluster $X_i$, i.e. how to compute $\cost_T(X_i \cup \{x\})$ given $\cost_T(X_i)$. The situation when we remove $x$ from a cluster is analogous. The updating of $n^i_j$ and $n_i$ is immediate (by a symbol with a hat $\hat{y}$ we denote the updated value of a variable $y$):
$$
\hat{n}^i_j = n^i_j + x_j \text{ and } \hat{n}_i = n_i + 1.
$$
In particular, $n^i_j$ only changes its value on these positions $j$, where $x_j$ is non-zero. 

Recalculation of $N^i_j$ is more complex, since it is calculated by using one of two formulas involved in \eqref{eq:Nij}, depending on the relation between $\frac{n^i_j}{n_i}$ and $T$. We consider four cases:
\begin{enumerate}
\item \label{case:LL} If $n^i_j \leq (n_i+1)T - 1$, then before and after the update we use the first formula of \eqref{eq:Nij}:
$$
\hat{N}^i_j = N^i_j + x_j.
$$
Moreover, this values changes only when $x_j = 1$. 
\item \label{case:RR} If $n^i_j > (n_i+1)T$, then before and after the update we use the second formula:
$$
\hat{N}^i_j = N^i_j + (1-x_j).
$$
It is changed only when $x_j = 0$. 
\item \label{case:L}If $x_j = 0$ and $n^i_j \in (n_iT, (n_i+1)T]$ then we switch from the second to the first formula and
$$
\hat{N}^i_j = n^i_j.
$$
Otherwise, it remains unchanged.
\item \label{case:R} If $x_j = 1$ and $n^i_j \in ((n_i+1)T - 1, n_iT]$ then we switch from the first to the second formula and
$$
\hat{N}^i_j = n_i - n^i_j.
$$ 
Otherwise, it remains unchanged.
\end{enumerate}
Due to the sparsity of $X$ there are only a few coordinates of $x$ satisfying $x_j = 1$. In consequence, the complexity of updates in the cases \ref{case:LL} and \ref{case:R} depends only on the number of non-zero bits in $X$. On the other hand, although $x_j = 0$ happens often, the situation when $n^i_j > n_iT$ is rare (for $T \geq \frac{1}{2}$), because $X$ is sparse. Since the cases \ref{case:RR} and \ref{case:L} cover a small number of coordinates, then they do not affect greatly on the complexity of the algorithm. Clearly, $S_i$ changes only if $N_i^j$ is changed as well.


Finally, to get the new cost of a cluster, we need to recalculate $\sum_{j= 1}^D N^i_j (-\log N^i_j)$. If we remember its old value $h(N^i_1,\ldots,N^i_D) = \sum_{j= 1}^D N^i_j (-\log N^i_j)$, then it is sufficient to update it on coordinates $j$ such that $N_i^j \neq \hat{N}_i^j$ by:
$$
h(\hat{N}^i_1,\ldots,\hat{N}^i_D)  = h(N^i_1,\ldots,N^i_D) - \sum_{j: N^i_j \neq \hat{N^i_j}} \left( \hat{N}^i_j (-\log \hat{N}^i_j) - N^i_j (-\log N^i_j)\right).
$$

We now analyze the computational cost of switching an element from one cluster to another. As discussed above, given $x \in X$ the recalculation of $N^i_j$, for $j=1,\ldots,D$, dominates the cost of updating any other quantity. Namely, we need to make updates on $c_i(x)$ coordinates, where:
\begin{equation} \label{eq:algoAnal}
\begin{array}{ll}
c_i(x) = c_{i,T}(x) & = |\{j: n^i_j \in ((n_i+1)T - 1, n_iT] \text{ and } x_j = 1\}|\\
& +|\{j: n^i_j \in (n_iT, (n_i+1)T] \text{ and } x_j = 0\}|\\
& +|\{j: n^i_j \leq (n_i+1)T - 1 \text{ and } x_j = 1 \}|\\
&+ |\{j: n^i_j > (n_i+1)T \text{ and } x_j = 0 \}|\\
& \leq |\{j:  x_j = 1 \}| + |\{j: n^i_j > (n_i+1)T -1  \}|\\
& \leq |\{j:  x_j = 1 \}| + |\{j: p^i_j > T - \frac{1-T}{n_i} \}|.\\
\end{array}
\end{equation}
Therefore, $c_i(x)$ is bounded from the above by the number of non-zero bits in $x$ and the number of coordinates where the probability $p^i_j$ of bit $1$ exceeds the threshold $T - \frac{1-T}{n_i}$. For $T=\frac{1}{2}$, this threshold equals $\frac{n_i-1}{2n_i}$, while for $T=1$ it attains value $1$ and, in consequence, $c_i(x)$ is exactly the number of coordinates with non-zero bits in $x$. It is also easy to see that $c_{i,T_1}(x) \geq c_{i, T_2}(x)$ if $\frac{1}{2} \leq T_1 < T_2$, i.e. the updates are faster for higher $T$.

\section{Experiments} \label{sec:experiments}

In this section we evaluate the performance of our algorithm and analyze its behavior in various clustering tasks. We compare its performance with related state-of-the-art methods. To denote our method we write $\sec(\beta,T)$, where $\beta$ and $T$ are the parameters of its cost function \eqref{def:secCost}.

We considered two types of Bernoulli mixture model. The first one is a classical mixture model, which relies on the maximum likelihood principle (ML) \cite{elmore2004estimating}. We used the R package ``mixtools'' \cite{benaglia2009mixtools} for its implementation. The second method is based on classification maximum likelihood (CML) \cite{li2004entropy}. While ML models every data point as a sample from a mixture of probability distributions, CML assigns every example to a single component. CML coincides with applying entropy as a clustering criterion \cite{celeux1991clustering}.

We also used two distance-based algorithms. The first one is $k$-medoids \cite{huang1998extensions}, which focuses on minimizing the average distance between data points and corresponding clusters' medoids (generalization of mean). We used R package ``cluster'' with Jaccard similarity measure\footnote{We also considered the Hamming and cosine distances, but the Jaccard distance provided the best results.}. We also considered Cluto software \cite{karypis2002cluto}, which is an optimized package for clustering large data sets. We ran the algorithm ``direct'' with a cosine distance function, which means that the package will calculate the final clustering directly, rather than bisecting the data multiple times. 

Since all of the aforementioned methods are non-deterministic and optimized in an iterative manner, each one was run 50 times with different initial partitions and the best result was selected according to the method's inner metric. 

\begin{table}[t]
\centering
\footnotesize
\caption{Summary of data sets used in experiments}
\label{tab:datasets}
\begin{tabular}{lcccc}
\hline

Dataset & Size & Dimensions & Avg. no. of non-zero bits & Classes\\ \hline
20newsgroups & 6997 & 26411 & 99.49 & 7 \\
farm-ads &     4143 & 54877 & 197.23 & 2 \\
questions &     5452 & 3029 & 4.04 & 6 \\
sentiment &    1000 & 2750 & 7.50 & 2 \\
SMS &          5574 & 7259 & 13.51 & 2 \\
chemical data &     3374 & 4860 & 65.45 & 5 \\
mushroom &     8124 & 119 & 21 & 2 \\
splice &     3190 & 287 & 60 & 2 \\
mnist    &     70000 & 784 & 150.08 & 10 \\
\hline
\end{tabular}
\end{table}

\subsection{Quality of clusters}

In this experiment we evaluated our method over various binary data sets, summarized in Table \ref{tab:datasets}, and compared its results with related methods listed at the beginning of this section. Since we considered classification data sets, we compared obtained clustering with reference partition. Their agreement was measured by Adjusted Rand Index (ARI), which is a well-known external validation index \cite{hubert1985comparing}. It attains maximal value 1 for identical partitions, while for random clustering\footnote{ARI might take negative values, in the case when produced partition is less compatible with reference grouping than a random assignment.} ARI = 0.

We used five text data sets: 20newsgroups, Farm-ads, SMS Spam Collection, Sentiment Labeled Sentences retrieved from UCI repository \cite{asuncionNewman2007} and Questions dataset taken from \cite{li2002learning}. Each data set was encoded in a binary form with use of the set-of-words representation. Set-of-words is one of the simplest vector representations of text. Given a dictionary of words, a document (or sentence) is represented as a binary vector, where coordinates indicate the presence or absence of words from a dictionary.

We considered a real data set containing chemical compounds acting on 5-HT$_{1A}$ receptor ligands \cite{warszycki13}. This is one of the proteins responsible for the regulation of the central nervous system. This data set was manually labeled by the expert in a hierarchical form. We narrowed that those classification tree down to 5 classes: tetralines, alkylamines, piperidines, aimides and other piperazines. Each compound was represented by its Klekota-Roth fingerprint, which encodes 4860 chemical patterns in a binary vector \cite{yap2011padel}.

We also took a molecular biology data set (splice), which describes primate splice-junction gene sequences. Moreover, we used data set containing mushrooms described in terms of physical characteristics, where the goal is to predict whether a mushroom is poisonous or edible. Both data sets were selected from UCI repository. Finally, we evaluated all methods on the MNIST data set \cite{lecun1998gradient}, which is a collection of handwritten digits made into black-and-white images. 

\begin{table}[t]
\footnotesize
\centering
\caption{Adjusted Rand Index of considered methods.}
\label{tab:ARI}
\begin{tabular}{lrrrrrrrr}
\hline
Data set & $\sec(0,\frac{1}{2})$ & $\sec(0,1)$ & $k$-medoids & Cluto & CML & ML \\ \hline
20newsgroups & 0.6337 & 0.8293 & 0.0012 & 0.7509 & 0.0991 & 0.0948 \\
farm-ads &     0.1958 & 0.2664 & 0.0192 & 0.2565 & 0.0468 & 0.0552 \\
questions &     0.0622 & 0.0622 & 0.0363 & 0.0926 & 0.0087 & 0.0274 \\
sentiment &    0.0667 & 0.0667 & 0.0153 & 0.0571 & 0.0064 & 0.0198 \\
SMS &          0.5433 & 0.5433 & 0.1081 & 0.5748 & 0.3063 & 0.3133 \\
chemical &     0.3856 & 0.4281 & 0.3724 & 0.3841 & 0.4472 & 0.4041 \\
mushroom &     0.6354 & 0.6275 & 0.626  & 0.6229 & 0.6354 & 0.6120 \\
splice &     0.7216 & 0.2607 & 0.1071 & 0.1592 & 0.4885 & 0.2442 \\
mnist    &     0.4501 & 0.395  & 0.3612 & 0.283  & 0.4277 & 0.4171 \\
\hline
\end{tabular}
\end{table}

All methods were run with the correct number of groups. Since the expected number of groups was given, {\sec} was run with $\beta = 0$ to prevent from clusters reduction. We examined its two parametrizations: (a) $T = \frac{1}{2}$, where a cluster representative is taken as the most probable point; (b) $T = 1$, where a representative is a zero vector.

The results presented in Table \ref{tab:ARI} shows significant disproportions between two best performing methods ({\sec} and Cluto) and other examined algorithms. The highest differences can be observed in the case of 20newsgroups, farm-adds and SMS data sets. In the case of the questions and sentiment data sets, neither method showed results significantly better than a random partitioning. Let us observe that these sets are extremely sparse, which could make the appropriate grouping of their examples very difficult. For the mushroom example, on the other hand, all methods seemed to perform equally good. Slightly higher differences can be observed on MNIST and chemical data sets, where ML and CML obtained good results. Finally, {\sec} with $T = \frac{1}{2}$ significantly outperformed other methods for splice.

Although {\sec}, ML and CML focus on optimizing similar cost functions, they use different algorithms, which could be the main reason for differences in their results. {\sec} applies an on-line Hartigan procedure, which updates clusters parameters at every switch, while ML and CML are based on EM algorithm and perform updates after the entire iteration. On-line updates allow for better model fitting and, in consequence, lead to finding better local minimums. This partially explains the more accurate clustering results of {\sec} compared to related  mixture models.

\begin{figure}[t]
\centering
\begin{tabular}{l|llllllllll}
&\includegraphics{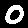} & \includegraphics{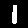} &
\includegraphics{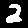} & \includegraphics{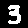} &
\includegraphics{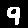} & \includegraphics{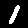} &
\includegraphics{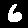} & \includegraphics{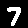} &
\includegraphics{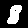} & \includegraphics{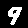}\\
\hline
0& 5634 & 6 & 12 & 211 & 15 & 24 & 182 & 4 & 778 & 37\\
1& 0 & 4140 & 9 & 5 & 2 & 3682 & 10 & 5 & 22 & 2\\
2& 91 & 616 & 4706 & 159 & 92 & 449 & 384 & 148 & 321 & 24\\
3& 48 & 613 & 260 & 4669 & 149 & 118 & 55 & 94 & 1083 & 52\\
4& 15 & 237 & 24 & 2 & 3444 & 208 & 246 & 30 & 65 & 2553\\
5& 113 & 270 & 11 & 2359 & 349 & 725 & 153 & 26 & 1264 & 1043\\
6& 110 & 478 & 111 & 33 & 34 & 227 & 5548 & 7 & 323 & 5\\
7& 25 & 289 & 24 & 3 & 850 & 457 & 4 & 4937 & 47 & 657\\
8& 70 & 451 & 57 & 1340 & 299 & 319 & 64 & 59 & 3798 & 377\\
9& 54 & 313 & 13 & 111 & 3224 & 112 & 15 & 288 & 59 & 2769\\
\end{tabular}
\caption{Confusion matrix and clusters representatives returned by applying {\sec} to the MNIST data set. Rows correspond to reference digits, while columns correspond to clusters produced by {\sec}.}
\label{fig:mnist}
\end{figure}

To further illustrate the effects of {\sec} we present its detailed results obtained on the MNIST data set. Figure \ref{fig:mnist} shows a confusion matrix and clusters representatives (first row) produced by {\sec} with $T=\frac{1}{2}$. It is clear that most of the clusters representatives resemble actual hand-drawn digits. It can be seen that {\sec} had trouble distinguishing between the digits 4 and 9, mixing them up a bit in their respective clusters. The digit 5 also could not be properly separated, resulting in its scatter among other clusters. The digit 1 occupied two separate clusters, once for being written vertically and once for being written diagonally. Nevertheless, this example showed that {\sec} is able to find reasonable clusters representatives that reflect their content in a strictly unsupervised way.

To summarize the results, we ranked the methods on each data set (the best performing method got rank 1, second best got rank 2, etc.). Figure \ref{fig:ranking} presents a box plot of ranks averaged over all data sets. The vertical lines show the range of the ranks, while the horizontal line in the middle denotes the median. It can be seen that both variants of {\sec} were equally good and outperformed other methods. Although the median rank of cluto was only slightly worse, its variance was significantly higher. This means that this model was not well suited for many data sets.

\begin{figure}[t]
\centering
\includegraphics[width=0.6\textwidth]{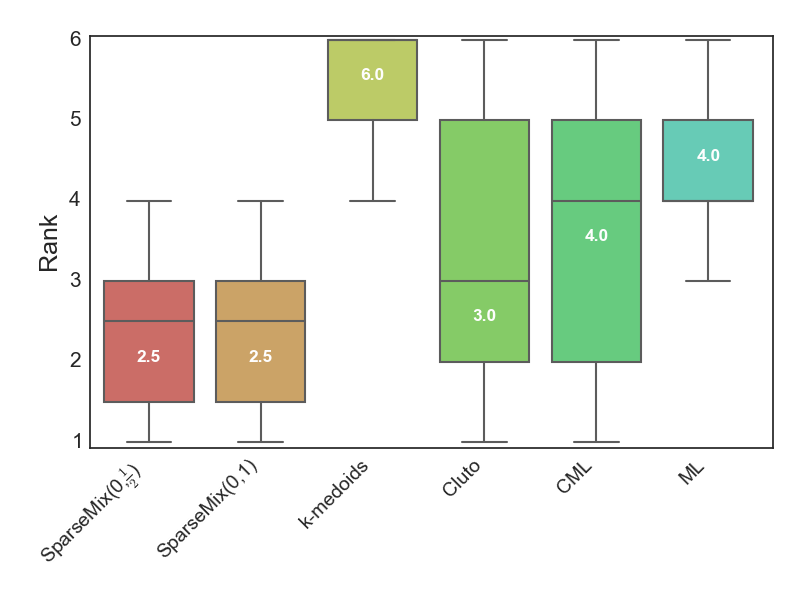}
\caption{Ranking of examined methods averaged over all data sets.}
\label{fig:ranking}
\end{figure}

\subsection{Time comparison}

In real-world applications, the clustering algorithm has to process large portions of data in a limited amount of time. In consequence, high computational complexity may disqualify a method from practical usage. In this experiment we focus on comparing the evaluation time of our algorithm with other methods. We tested the dependence on the number of data points as well as on the number of attributes. For the illustration, we considered the chemical data set from previous subsection.

\begin{figure}[t]
\centering
\includegraphics[width=0.9\textwidth]{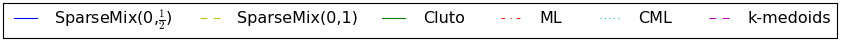}
\subfigure[]{\label{fig:timeNumFarm}\includegraphics[width=0.48\textwidth]{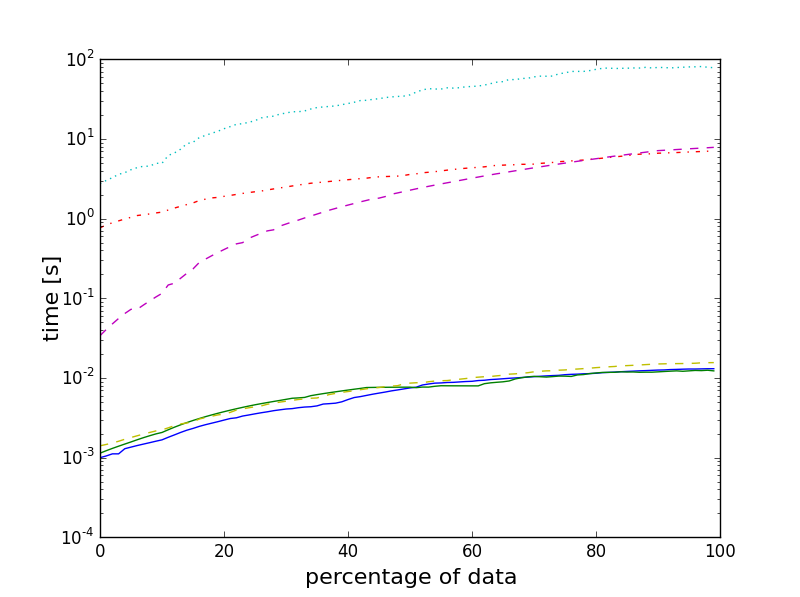}}
\subfigure[]{\label{fig:timeDimFarm}\includegraphics[width=0.48\textwidth]{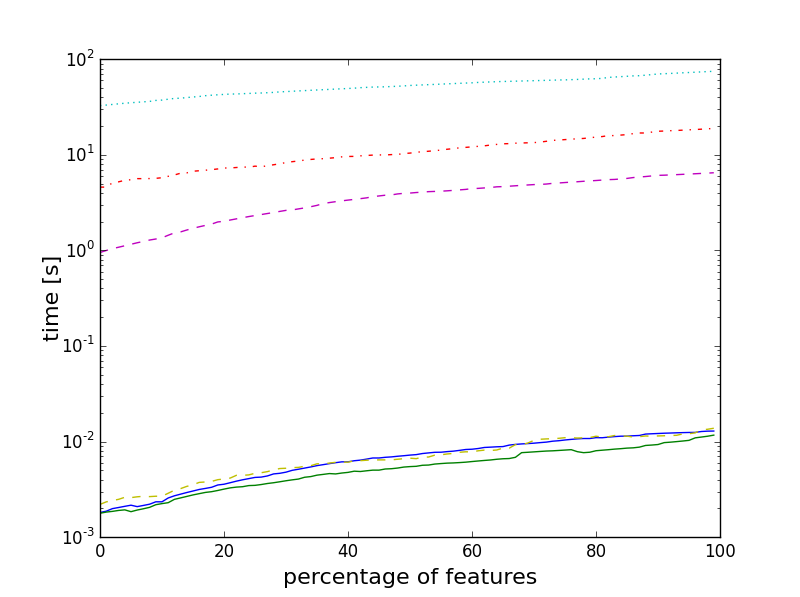}}
\caption{The running times with respect to the number of points (\ref{fig:timeNumFarm}) and to the number attributes (\ref{fig:timeDimFarm}) presented in logarithmic scale.}
\label{fig:time}
\end{figure}

In the first scenario, we randomly selected a subset of data containing a given percentage of instances, while in the second simulation, we chose a given percentage of attributes. The clustering algorithms were run on such prepared subsets of data. The results presented in Figure \ref{fig:time} show that both versions of {\sec} were as fast as the Cluto package, which is an optimized software for processing large data sets. The other algorithms were significantly slower. It might be caused both by a specific clustering procedure as well as by an inefficient programming language used for their implementations. 

\begin{figure}[t]
\centering
\subfigure[$T=0.5$]{\label{fig:energy1} \includegraphics[width=0.48\textwidth]{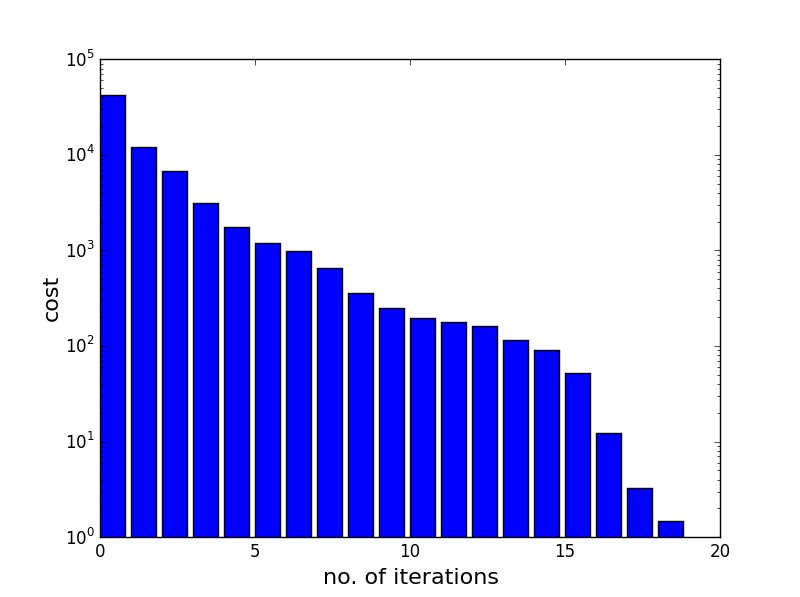}}
\subfigure[$T=1$]{\label{fig:energy2} \includegraphics[width=0.48\textwidth]{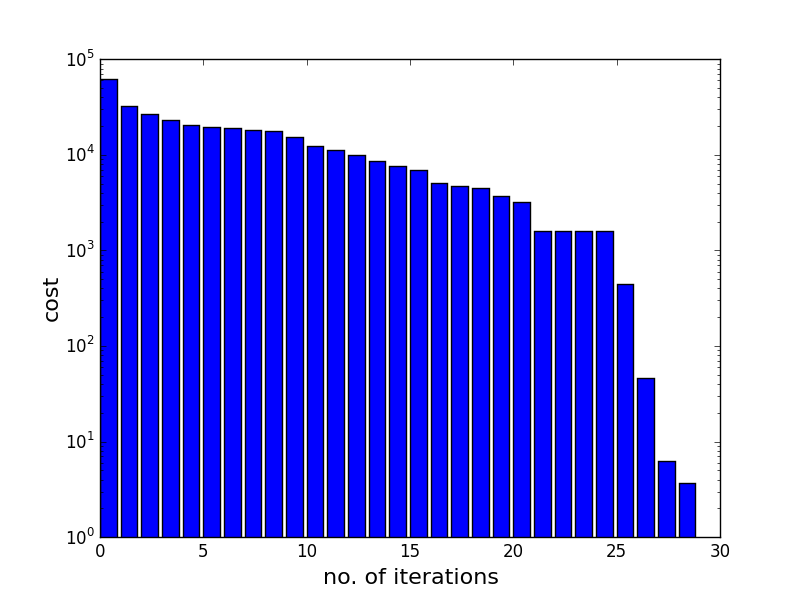}}
\caption{The difference between the cost in each iteration and the cost of the final clustering of {\sec} with $T=0.5$ (\ref{fig:energy1}) and $T=1$ (\ref{fig:energy1}) given in logarithmic scale. }
\label{fig:energy}
\end{figure}

The interesting thing is that {\sec} with $T=\frac{1}{2}$ was often slightly faster than {\sec} with $T = 1$, which at first glance contradicts the theoretical analysis of our algorithm. To investigate this observation, we counted the number of iterations needed for convergence of both methods. It is clear from the Figure \ref{fig:energy} that {\sec} with $T=\frac{1}{2}$ needed less iterations to find a local minimum than with $T=1$, which fully explains the relation between their running times. {\sec} with $T=\frac{1}{2}$ needed less than 20 iterations to converge. Since the scale of the graph is logarithmic, the differences in its cost decreased exponentially. Such a fast convergence follows from that fact that the {\sec} cost function can be optimized by applying on-line Hartigan algorithm (this is computationally impossible to use an on-line strategy for CML or ML models).

\subsection{Clustering stability}

In this experiment, we examined the stability of considered clustering algorithms upon the changing of the data. More precisely, we tested whether a method was able to preserve clustering results when some data instances or attributes were removed. In practical application high stability of an algorithm can be used to speed up the clustering procedure. If a method does not change its result using a lower number of instances or attributes, we can safely perform clustering on a reduced data set and assign the remaining instances to the nearest clusters. We again used the chemical data set for this experiment. In this simulation we only ran {\sec} with $T = \frac{1}{2}$ (our preliminary studies showed that parameter $T$ does not visibly influence overall results).

First, we investigated the influence of the number of instances on the clustering results. For this purpose, we performed the clustering of the whole data set $X$ and randomly selected $p$ percentage of its instances $X^p$ (we considered $p = 0.1, 0.2, \ldots, 0.9$). Stability was measured by calculating ARI between the clusters $X_1^p,\ldots,X_k^p$ created from the selected fraction of data $X^p$ and from the whole data set (restricted to the same instances), i.e. $(X_1 \cap X^p), \ldots, (X_k \cap X^p)$. To reduce the effect of randomness, this procedure was repeated 5 times and the final results were averaged. The results presented in Figure \ref{fig:stabRev1} show that for a small number of data points Cluto gave the highest stability, but as the number of instances grows {\sec} performed better. 

In the second part of the experiment, we examined how the clustering results changed when a smaller number  of attributes were taken into account. The procedure was analogical to the previous one: we compared the clustering results obtained on the whole data set with the ones produced on data set with randomly selected $p$ percentage of attributes (as before we considered $p = 0.1, 0.2, \ldots, 0.9$). One can observe in Figure \ref{fig:stabRev2} that {\sec} obtained the highest stability on all subsets of data. The performance of Cluto was significantly worse than previously -- in particular, ML showed higher stability.

\begin{figure}[t]
\centering
\includegraphics[width=0.8\textwidth]{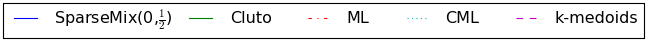}
\subfigure[]{\label{fig:stabRev1} \includegraphics[width=0.48\textwidth]{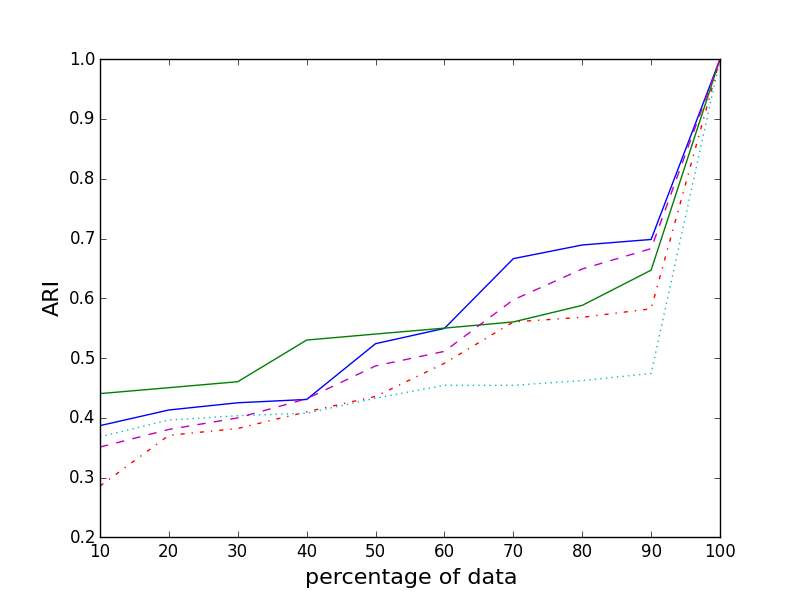}}
\subfigure[]{\label{fig:stabRev2} \includegraphics[width=0.48\textwidth]{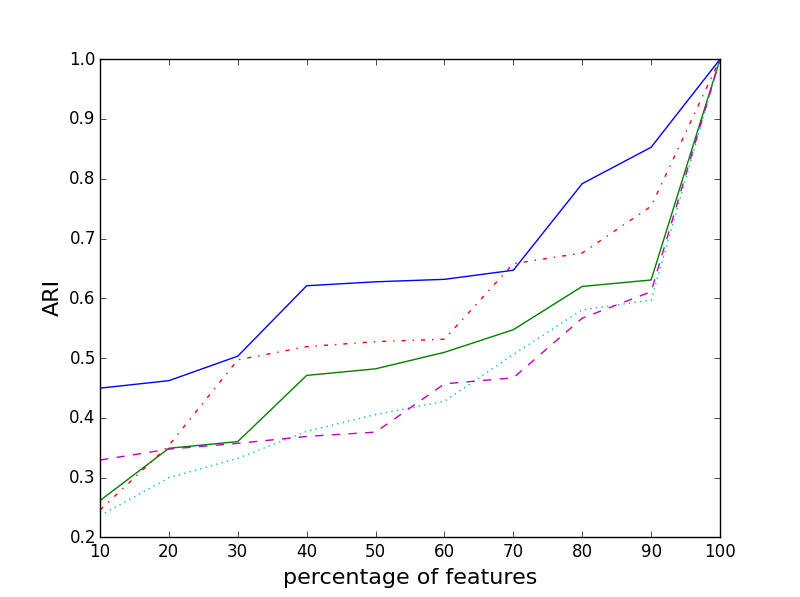}}
\caption{Compatibility between clustering results obtained on the whole data set and its fixed percentage.}
\label{fig:stabRev}
\end{figure}

\subsection{Sensitivity to imbalanced data}

In the following section, we examined sensitivity of the clustering algorithms to data imbalance. This extends theoretical analysis presented in Example \ref{ex:reduction}. 

First, we examined whether the algorithm is able to detect clusters of different sizes. For this purpose, we considered a data set $X \subset \{ 0, 1 \}^D$, for $D = 100$ and $|X| = 1000$, generated from a distribution 
$$
\omega P(p, \alpha, d) + (1-\omega)P(p, 1-\alpha, d),
$$ 
where $p = 0.1$, $\alpha = 0.05$ and $d = D/2$ were fixed and $\omega$ changed from $0$ to $1$. We refer the reader to Example \ref{ex:reduction} for the definition of distribution $P$ and its interpretation. The mixing parameter $\omega$ induces the fraction of examples produced by these two sources. We would expect that a clustering method will be able to discover true distributions, so the resulting sizes of the clusters will be, roughly, $\omega |X|$ and $(1-\omega)|X|$. However, as $\omega$ approaches either to 0 or 1, the data becomes very imbalanced, which makes the task of separating them more difficult. We considered {\sec} with $\beta = 0$ and $\beta = 1$ to account for different costs of maintaining clusters (our preliminary studies showed that parameter $T$ does not visibly influence overall results and thus we used $T=\frac{1}{2}$).

\begin{figure}[t]\centering
\includegraphics[width=0.9\textwidth]{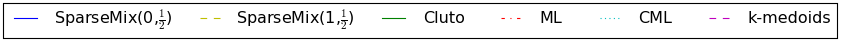}
\subfigure[]{\label{fig:bal1}\includegraphics[width=0.48\textwidth]{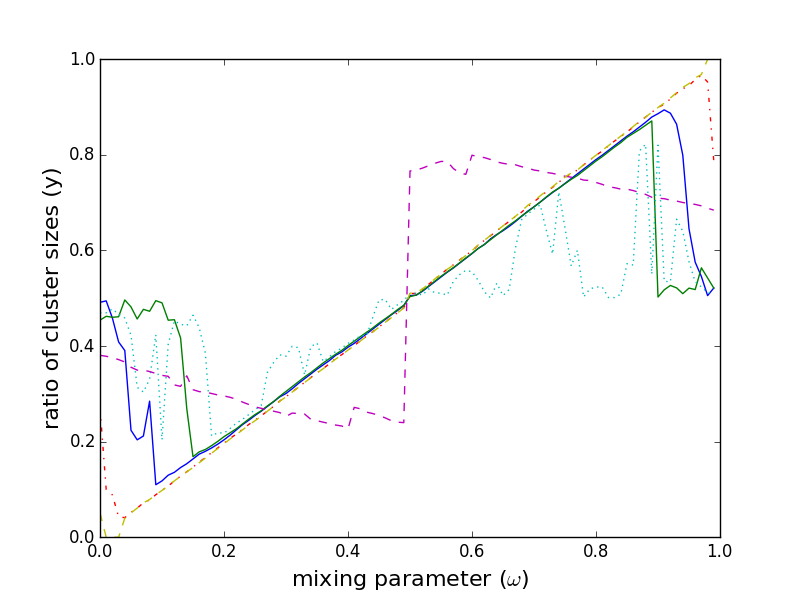}}
\subfigure[]{\label{fig:bal2}\includegraphics[width=0.48\textwidth]{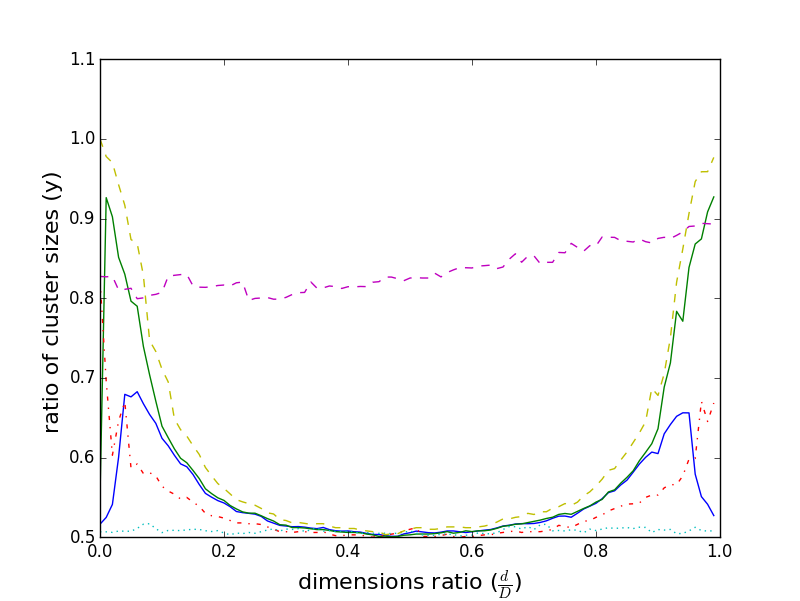}}
\caption{The ratio of cluster sizes for data sets generated from imbalanced sources: we varied the number of instances generated from each source \ref{fig:bal1} and the number of attributes characteristic for each source \ref{fig:bal2}.}
\end{figure}

Figure \ref{fig:bal1} reports the fraction of the data that belongs to the first cluster. The optimal solution is $y = \omega$. We can see that the $k$-medoids method did not respond to the changes in $\omega$. Other algorithms seemed to perform well on the mid section, but gradually steered off the optimal line as $\omega$ approached to 0 or 1. The highest robustness to imbalanced data was obtained by ML and {\sec} with $\beta = 1$ (cost of clusters identification was taken into account). If the cost of maintaining clusters is not considered ($\beta = 0$), then {\sec} tends to create more balanced groups. These results are consistent with a discussion outlined before Definition \ref{def:secCost} and in Example \ref{ex:reduction}.

In the second experiment, we investigated the influence of attributes imbalance on the clustering results. For this purpose we sampled a data set from the mixture of distributions given by:
$$
\frac{1}{2}P(p, \alpha, d) + \frac{1}{2}P(p, 1-\alpha, d),
$$
where $p = 0.1$, $\alpha = 0.05, |X| = 1000$ and $D = 100$ were constants, while $d$ ranged from 0 to $D$. When $d < D$, then the second source is identified by a smaller number of bits than the first one. Therefore, by changing the value of the parameter $d$ we scale the number of features characteristic for components. This time we expect that the clusters will remain equally-sized regardless of the parameter $d$.

Figure \ref{fig:bal2} presents the fraction of data that belongs to the first cluster (perfect solution is given by $y = \frac{1}{2}$).   It can be observed that {\sec} with $\beta = 1$ was very sensitive to attributes imbalance. According to conclusion given in Example \ref{ex:reduction}, the cost of encoding elements within a cluster is outweighed by the cost of clusters identification, as $\alpha \to 0$ (or $1$), which results in the reduction of the lighter group. Since the data is sampled from an underlying distribution and {\sec} flows to a local minimum, some attempts result in creating one group, while the others produce two clusters, which explains why the corresponding line is not equal to 1, for $\alpha < 0.2$. This effect was not apparent when {\sec} used $\beta = 0$, because there was no cost of creating an additional cluster. Its results were comparable to ML and CML, which also do not use any cost of clusters identification.

\section{Conclusion}

In this paper, we proposed {\sec}, a new approach for clustering of sparse high dimensional binary data. Our results showed that {\sec} is not only more accurate than related model-based clustering algorithms, but also significantly faster. Its evaluation time is comparable to algorithms implemented in the Cluto package, the software optimized for processing large data sets, but its clusters quality is better. {\sec} provides a description of each cluster by its representative and the dispersion from this representative. Experimental results demonstrated that representatives obtained for the MNIST data set provide high resemblance with original representatives of handwritten digits. The model was theoretically analyzed.

\appendix

\section{Proof of Theorem \ref{thm:compression}} \label{app:compression}

We will show that $\cost_{T_2}(X) - \cost_{T_1}(X) \geq 0$. We have:
\begin{align*}
\cost_{T_2}(X) - \cost_{T_1}(X) & = \sum_{i=1}^D \left(q_i(T_1)\log Q_i(T_1) - q_i(T_2) \log Q_i(T_2) \right)\\
& = \sum_{i: p_i \leq T_1} \left(p_i \log \frac{p_i}{Z(T_1)} - p_i \log \frac{p_i}{Z(T_2)} \right)\\ 
& + \sum_{i: T_1 \leq p_i \leq T_2} \left((1-p_i) \log \frac{1-p_i}{Z(T_1)} - p_i \log \frac{p_i}{Z(T_2)} \right)\\
& + \sum_{i: p_i \geq T_2} \left((1-p_i) \log \frac{1-p_i}{Z(T_1)} - (1-p_i) \log \frac{1-p_i}{Z(T_2)} \right)\\
& =  \log \frac{Z(T_2)}{Z(T_1)} \left(\sum_{i: p_i \leq T_1}  p_i  + \sum_{i: p_i \geq T_2} (1-p_i)  \right)\\
& + \sum_{i: T_1 \leq p_i \leq T_2} (1-p_i) \log (1-p_i) - p_i \log p_i \\
& + \sum_{i: T_1 \leq p_i \leq T_2} p_i \log Z(T_2) - (1-p_i) \log Z(T_1).\\
\end{align*}

Observe that $Z(T_1) \leq Z(T_2)$ and thus $\log \frac{Z(T_2)}{Z(T_1)} \geq 0$. Consequently,
\begin{align*}
\cost_{T_2}(X) - \cost_{T_1}(X) \geq  \sum_{i: T_1 \leq p_i \leq T_2}&  \left((1-p_i) \log (1-p_i) - p_i \log p_i \right. \\
& \left. + p_i \log Z(T_2) - (1-p_i) \log Z(T_1) \right).
\end{align*}

The above expression is non-negative if only the function:
$$
f(p) = p (-\log p) - (1-p) (-\log(1-p)) + p \log Z(T_2) - (1-p) \log Z(T_1)
$$
is non-negative for every $T_1 \leq p \leq T_2$.

A derivative of $f$ equals:
\begin{align*}
f'(p) & = -\log p(1-p) + \log (Z(T_1) Z(T_2)) - 2\\
& = -\log p(1-p) + \log \frac{Z(T_1)Z(T_2)}{4}.
\end{align*}
It is greater than zero when:
$$
\log \frac{Z(T_1)Z(T_2)}{4} \geq \log p(1-p),
$$
which simplifies to:
$$
\frac{Z(T_1)Z(T_2)}{4} \geq p(1-p).
$$

Since $Z(T_2) \geq Z(T_1) \geq 1$, then $Z(T_1)Z(T_2) \geq 1$. Moreover, due to the fact that $p(1-p) \leq \frac{1}{4}$, for $p \in [0,1]$, we have
$$
\frac{Z(T_1)Z(T_2)}{4} \geq \frac{1}{4} \geq p(1-p),
$$
which means that for every $p$ satisfying $T_1\leq p \leq T_2$ the function $f$ is nondecreasing. Finally,
$$
f(\frac{1}{2}) = \frac{1}{2} \log \frac{Z(T_2)}{Z(T_1)} \geq 0,
$$
and consequently, $f(p) \geq 0$, for $\frac{1}{2} \leq T_1\leq p \leq T_2 \leq 1$. This means that the best compression is achieved for $T = \frac{1}{2}.$

\section{Clusters reduction - details of Example \ref{ex:reduction}} \label{app:reduction}

We compare the cost of using a single cluster for all instances with the cost of splitting the data into two optimal groups (first $\omega |X|$ examples are assigned to the first group while the remaining instances are assigned to the second cluster). For the convenience of calculations, we define the function:
$$
D(x,d) := x d + (1-x) (D-d),
$$

The conditional probability $\Q_i^1$ that the $i$-th position holds a non-zero value in the first cluster equals:
$$
Q^1_i =  \left\{ 
\begin{array}{ll}
 \frac{\alpha}{D(\alpha,d)}, & j =1,\ldots,d,\\
 \frac{1-\alpha}{D(\alpha,d)}, & j = d+1,\dots,D.
\end{array}
\right.
$$
while for the second group:
$$
Q^2_i =  \left\{ 
\begin{array}{ll}
 \frac{1-\alpha}{D(1-\alpha,d)}, & j =1,\ldots,d,\\
 \frac{\alpha}{D(1-\alpha,d)}, & j = d+1,\dots,D.
\end{array}
\right.
$$
Then, the cost of using two clusters equals:
\begin{multline} \label{eq:cost:2}
\cost(X_1, X_2)  = \omega (-\log \omega) + (1-\omega) (-\log (1-\omega) )\\
 - \omega \left(d \alpha p \log  \frac{\alpha}{D(\alpha,d)} + (D-d) (1-\alpha) p \log \frac{1-\alpha}{D(\alpha,d)} \right)\\
 - (1-\omega) \left(d (1-\alpha) p \log \frac{1-\alpha}{D(1-\alpha,d)} +(D-d) \alpha p\log \frac{\alpha}{D(1-\alpha,d)} \right) \\
 =p \left( \omega D(\alpha,d) \log D(\alpha,d) + (1-\omega) D(1-\alpha,d) \log D(1-\alpha,d) \right.\\
\left. - \alpha D(\omega,d) \log \alpha - (1-\alpha) D(1-\omega,d) \log (1-\alpha) \right) + h(\omega,1-\omega).
\end{multline}

To calculate the cost of one cluster, let us put $\beta = \omega \alpha+ (1-\omega) (1-\alpha)$. Then, $(1-\beta) = \omega (1-\alpha) + (1-\omega) \alpha$ and the conditional probability $Q_i$ is given by
$$
Q_i =  \left\{ 
\begin{array}{ll}
\frac{\beta}{D(\beta,d)}, & j =1,\ldots,d,\\
\frac{1-\beta}{D(\beta,d)}, & j = d+1,\dots,D.
\end{array}
\right.
$$
The cost of one cluster can be written as follows:
\begin{multline} \label{eq:cost:1}
\cost(X)  = -d p \beta \log \frac{\beta}{D(\beta,d)} -(D-d) p (1-\beta) \log \frac{1-\beta}{D(\beta,d)}\\
 =p (D(\beta,d) \log D(\beta,d) -d \beta \log \beta -(D-d) (1-\beta) \log (1-\beta)).\\
\end{multline}

It is more profitable to use one cluster instead of two if \eqref{eq:cost:1} is lower than \eqref{eq:cost:2}. Since it is difficult to analyze this relation in general, we consider three special cases:
\begin{enumerate}
\item {\bf Dimensions are balanced.} We fix the dimension parameter $d = \frac{1}{2} D$. Then $D(\alpha, d) = D(\omega,d) =D(\beta,d)= d$ and the formula \eqref{eq:cost:2} simplifies to:
$$
\cost(X_1,X_2) = pd \left( \log d + h(\alpha, 1-\alpha) \right) + h(\omega,1-\omega),
$$
while \eqref{eq:cost:1} equals:
$$
\cost(X)  = pd \left( h(\beta, 1-\beta) + \log d \right).
$$

\item {\bf Sources are balanced.} If we fix the mixing proportion $\omega= \frac{1}{2}$ then the cost of two clusters is:
\begin{multline*}
\cost(X_1,X_2) = -\frac{1}{2}p \left( h(D(\alpha,d), D(1-\alpha, d)) + D h(\alpha,1-\alpha) \right) + \log 2
\end{multline*}
and for one cluster we have
$$
\cost(X) =  \frac{1}{2} p D \log D.
$$

\item {\bf Both dimensions and sources are balanced.} For fixed $d = \frac{1}{2}D$ and $\omega = \frac{1}{2}$ the cost of two clusters is given by
$$
\cost(X_1,X_2) =d p (h(\alpha, 1 - \alpha) + \log d)+ \log 2,
$$
while for one cluster we have
$$
\cost(X) =p d \log D.
$$

\end{enumerate}

\section*{Acknowledgement}

This research was partially supported by the National Science Centre (Poland) grant no. 2014/13/N/ST6/01832 and grant no. 2015/19/B/ST6/01819.

\section*{References}

\bibliographystyle{plain}      
\bibliography{references}   

\begin{thebibliography}{10}

\bibitem{asuncionNewman2007}
D.J.~Newman A.~Asuncion.
\newblock {UCI} machine learning repository, 2007.

\bibitem{bai2011novel}
Liang Bai, Jiye Liang, Chuangyin Dang, and Fuyuan Cao.
\newblock A novel attribute weighting algorithm for clustering high-dimensional
  categorical data.
\newblock {\em Pattern Recognition}, 44(12):2843--2861, 2011.

\bibitem{baker1998distributional}
L~Douglas Baker and Andrew~Kachites McCallum.
\newblock Distributional clustering of words for text classification.
\newblock In {\em Proceedings of the 21st annual international ACM SIGIR
  conference on Research and development in information retrieval}, pages
  96--103. ACM, 1998.

\bibitem{barbara2002coolcat}
Daniel Barbar{\'a}, Yi~Li, and Julia Couto.
\newblock Coolcat: an entropy-based algorithm for categorical clustering.
\newblock In {\em Proceedings of the eleventh international conference on
  Information and knowledge management}, pages 582--589. ACM, 2002.

\bibitem{benaglia2009mixtools}
Tatiana Benaglia, Didier Chauveau, David Hunter, and Derek Young.
\newblock mixtools: An r package for analyzing finite mixture models.
\newblock {\em Journal of Statistical Software}, 32(6):1--29, 2009.

\bibitem{bouguila2010multivariate}
Nizar Bouguila.
\newblock On multivariate binary data clustering and feature weighting.
\newblock {\em Computational Statistics \& Data Analysis}, 54(1):120--134,
  2010.

\bibitem{bouguila2009discrete}
Nizar Bouguila and Walid ElGuebaly.
\newblock Discrete data clustering using finite mixture models.
\newblock {\em Pattern Recognition}, 42(1):33--42, 2009.

\bibitem{cagnone2012factor}
Silvia Cagnone and Cinzia Viroli.
\newblock A factor mixture analysis model for multivariate binary data.
\newblock {\em Statistical Modelling}, 12(3):257--277, 2012.

\bibitem{celeux1991clustering}
Gilles Celeux and G{\'e}rard Govaert.
\newblock Clustering criteria for discrete data and latent class models.
\newblock {\em Journal of classification}, 8(2):157--176, 1991.

\bibitem{chan2004optimization}
Elaine~Y Chan, Wai~Ki Ching, Michael~K Ng, and Joshua~Z Huang.
\newblock An optimization algorithm for clustering using weighted dissimilarity
  measures.
\newblock {\em Pattern recognition}, 37(5):943--952, 2004.

\bibitem{chen2016soft}
Lifei Chen, Shengrui Wang, Kaijun Wang, and Jianping Zhu.
\newblock Soft subspace clustering of categorical data with probabilistic
  distance.
\newblock {\em Pattern Recognition}, 51:322--332, 2016.

\bibitem{cover2012elements}
Thomas~M Cover and Joy~A Thomas.
\newblock {\em Elements of information theory}.
\newblock John Wiley \& Sons, 2012.

\bibitem{dhillon2003information}
Inderjit~S Dhillon and Yuqiang Guan.
\newblock Information theoretic clustering of sparse cooccurrence data.
\newblock In {\em Data Mining, 2003. ICDM 2003. Third IEEE International
  Conference on}, pages 517--520. IEEE, 2003.

\bibitem{dos2015categorical}
Tiago~RL dos Santos and Luis~E Z{\'a}rate.
\newblock Categorical data clustering: What similarity measure to recommend?
\newblock {\em Expert Systems with Applications}, 42(3):1247--1260, 2015.

\bibitem{elmore2004estimating}
Ryan~T Elmore, Thomas~P Hettmansperger, and Hoben Thomas.
\newblock Estimating component cumulative distribution functions in finite
  mixture models.
\newblock {\em Communications in Statistics-Theory and Methods},
  33(9):2075--2086, 2004.

\bibitem{ewing2006novel}
Todd Ewing, J~Christian Baber, and Miklos Feher.
\newblock Novel 2d fingerprints for ligand-based virtual screening.
\newblock {\em Journal of chemical information and modeling}, 46(6):2423--2431,
  2006.

\bibitem{franti2003classification}
Pasi Fr{\"a}nti, Mantao Xu, and Ismo K{\"a}rkk{\"a}inen.
\newblock Classification of binary vectors by using $\delta$sc distance to
  minimize stochastic complexity.
\newblock {\em Pattern Recognition Letters}, 24(1):65--73, 2003.

\bibitem{gollini2014mixture}
Isabella Gollini and Thomas~Brendan Murphy.
\newblock Mixture of latent trait analyzers for model-based clustering of
  categorical data.
\newblock {\em Statistics and Computing}, 24(4):569--588, 2014.

\bibitem{graham2006unsupervised}
Michael~W Graham and David~J Miller.
\newblock Unsupervised learning of parsimonious mixtures on large spaces with
  integrated feature and component selection.
\newblock {\em IEEE Transactions on Signal Processing}, 54(4):1289--1303, 2006.

\bibitem{hartigan1979algorithm}
John~A Hartigan and Manchek~A Wong.
\newblock Algorithm as 136: A k-means clustering algorithm.
\newblock {\em Applied statistics}, pages 100--108, 1979.

\bibitem{huang1998extensions}
Zhexue Huang.
\newblock Extensions to the k-means algorithm for clustering large data sets
  with categorical values.
\newblock {\em Data mining and knowledge discovery}, 2(3):283--304, 1998.

\bibitem{hubert1985comparing}
Lawrence Hubert and Phipps Arabie.
\newblock Comparing partitions.
\newblock {\em Journal of classification}, 2(1):193--218, 1985.

\bibitem{juan2002use}
Alfons Juan and Enrique Vidal.
\newblock On the use of bernoulli mixture models for text classification.
\newblock {\em Pattern Recognition}, 35(12):2705--2710, 2002.

\bibitem{karypis2002cluto}
George Karypis.
\newblock Cluto-a clustering toolkit.
\newblock Technical report, DTIC Document, 2002.

\bibitem{klekota2008chemical}
Justin Klekota and Frederick~P Roth.
\newblock Chemical substructures that enrich for biological activity.
\newblock {\em Bioinformatics}, 24(21):2518--2525, 2008.

\bibitem{langseth2009latent}
Helge Langseth and Thomas~D Nielsen.
\newblock Latent classification models for binary data.
\newblock {\em Pattern Recognition}, 42(11):2724--2736, 2009.

\bibitem{lecun1998gradient}
Yann LeCun, L{\'e}on Bottou, Yoshua Bengio, and Patrick Haffner.
\newblock Gradient-based learning applied to document recognition.
\newblock {\em Proceedings of the IEEE}, 86(11):2278--2324, 1998.

\bibitem{li2005general}
Tao Li.
\newblock A general model for clustering binary data.
\newblock In {\em Proceedings of the eleventh ACM SIGKDD international
  conference on Knowledge discovery in data mining}, pages 188--197. ACM, 2005.

\bibitem{li2004entropy}
Tao Li, Sheng Ma, and Mitsunori Ogihara.
\newblock Entropy-based criterion in categorical clustering.
\newblock In {\em Proceedings of the twenty-first international conference on
  Machine learning}, page~68. ACM, 2004.

\bibitem{li2002learning}
Xin Li and Dan Roth.
\newblock Learning question classifiers.
\newblock In {\em Proceedings of the 19th international conference on
  Computational linguistics-Volume 1}, pages 1--7. Association for
  Computational Linguistics, 2002.

\bibitem{macqueen1967some}
James MacQueen et~al.
\newblock Some methods for classification and analysis of multivariate
  observations.
\newblock In {\em Proceedings of the fifth Berkeley symposium on mathematical
  statistics and probability}, volume~1, pages 281--297. Oakland, CA, USA.,
  1967.

\bibitem{mclachlan2004finite}
Geoffrey McLachlan and David Peel.
\newblock {\em Finite mixture models}.
\newblock John Wiley \& Sons, 2004.

\bibitem{rissanen1985minimum}
Jorma Rissanen.
\newblock Minimum-description-length principle.
\newblock {\em Encyclopedia of statistical sciences}, 1985.

\bibitem{Strouse_DIB}
DJ~Strouse and David~J. Schwab.
\newblock The deterministic information bottleneck.
\newblock In {\em Proc. Conf. on Uncertainty in Artificial Intelligence (UAI)},
  pages 696--705, New York City, NY, June 2016.

\bibitem{tang2015}
Yang Tang, Ryan~P. Browne, and Paul~D. McNicholas.
\newblock Model based clustering of high-dimensional binary data.
\newblock {\em Computational Statistics \& Data Analysis}, 87:84--101, 2015.

\bibitem{tishby2000information}
Naftali Tishby, Fernando~C. Pereira, and William Bialek.
\newblock The information bottleneck method.
\newblock In {\em Proc. Allerton Conf. on Communication, Control, and
  Computing}, pages 368--377, Monticello, IL, September 1999.

\bibitem{vermunt2007multilevel}
Jeroen~K Vermunt.
\newblock Multilevel mixture item response theory models: an application in
  education testing.
\newblock {\em Proceedings of the 56th session of the International Statistical
  Institute. Lisbon, Portugal}, 2228, 2007.

\bibitem{wallace1968information}
Christopher~S Wallace and David~M Boulton.
\newblock An information measure for classification.
\newblock {\em The Computer Journal}, 11(2):185--194, 1968.

\bibitem{warszycki13}
Dawid Warszycki, Stefan Mordalski, Kurt Kristiansen, Rafa\l{} Kafel, Ingebrigt
  Sylte, Zdzis\l{}aw Chilmonczyk, and Andrzej~J. Bojarski.
\newblock A linear combination of pharmacophore hypotheses as a new tool in
  search of new active compounds - an application for 5-ht$_{1A}$ receptor
  ligands.
\newblock {\em PLoS ONE}, 8(12):e84510, 12 2013.

\bibitem{wen2002query}
Ji-Rong Wen, Jian-Yun Nie, and Hong-Jiang Zhang.
\newblock Query clustering using user logs.
\newblock {\em ACM Transactions on Information Systems}, 20(1):59--81, 2002.

\bibitem{yamamoto2015clustering}
Michio Yamamoto and Kenichi Hayashi.
\newblock Clustering of multivariate binary data with dimension reduction via l
  1-regularized likelihood maximization.
\newblock {\em Pattern Recognition}, 48(12):3959--3968, 2015.

\bibitem{yap2011padel}
Chun~Wei Yap.
\newblock Padel-descriptor: An open source software to calculate molecular
  descriptors and fingerprints.
\newblock {\em Journal of computational chemistry}, 32(7):1466--1474, 2011.

\bibitem{zhao2002evaluation}
Ying Zhao and George Karypis.
\newblock Evaluation of hierarchical clustering algorithms for document
  datasets.
\newblock In {\em Proceedings of the eleventh international conference on
  Information and knowledge management}, pages 515--524. ACM, 2002.

\end{thebibliography}

\end{document}